\documentclass[conference]{IEEEtran}
\IEEEoverridecommandlockouts
\usepackage{cite}
\usepackage{amsmath,amssymb,amsfonts}
\usepackage{amsthm}
\usepackage{algorithmic}
\usepackage{graphicx}
\usepackage{textcomp}
\usepackage{xcolor}
\usepackage{multirow}
\usepackage[caption=false]{subfig}

\newtheorem{theorem}{Theorem}
\newtheorem{proposition}{Proposition}

\newtheorem{definition}{Definition}

\begin{document}

\title{Implementing Ranking-Based Semantics in ConArg: a Preliminary Report}
\author{
	\IEEEauthorblockN{Stefano Bistarelli}
	\IEEEauthorblockA{\textit{Dep. of Mathematics and Computer Science} \\
		\textit{University of Perugia}\\
		Perugia, Italy \\
		stefano.bistarelli@unipg.it}
	\and
	\IEEEauthorblockN{Francesco Faloci}
	\IEEEauthorblockA{\textit{Dep. of Mathematics and Computer Science} \\
		\textit{University of Perugia}\\
		Perugia, Italy \\
		francesco.faloci@studenti.unipg.it}
	\and
	\IEEEauthorblockN{Carlo Taticchi}
	\IEEEauthorblockA{\textit{Dep. of Computer Science} \\
		\textit{Gran Sasso Science Institute}\\
		L'Aquila, Italy \\
		carlo.taticchi@gssi.it}
}

\maketitle

\begin{abstract}
	ConArg is a suite of tools that offers a wide series of applications for dealing with argumentation problems. In this work, we present the advances we made in implementing a ranking-based semantics, based on computational choice power indexes, within ConArg. Such kind of semantics represents a method for sorting the arguments of an abstract argumentation framework, according to some preference relation. The ranking-based semantics we implement relies on Shapley, Banzhaf, Deegan-Packel and Johnston power index, transferring well know properties from computational social choice to argumentation framework ranking-based semantics.
\end{abstract}


%introdurre acceptable arguments
%dire cosa è una extension 

\section{Introduction}

Argumentation Theory is a field of Artificial Intelligence that provides formalisms for reasoning with conflicting information, and many different areas, ranging from healthcare to systems optimisation, make use of notions coming from the research in computational argumentation. To give just a few examples, from a social science perspective, argumentation is used to model the human behaviour in the simulation of a rational population\cite{Gabbriellini2013AbstractAF}; the problem of the breast cancer recurrence prediction is addressed in~\cite{DBLP:conf/cbms/LongoKH12} through an argumentative process that represents clinical evidence; finally, the authors in~\cite{DBLP:conf/aaai/CyrasLMT19} define argumentative explanations for why a schedule is (or is not) feasible and efficient.

Arguments from a knowledge base are modelled by Dung as nodes in a directed graph, called Abstract Argumentation Framework (AF in short), where edges represent a binary attack relation between arguments. In a framework, it is possible to select the sets of arguments that are not in conflict with each other. Many \textit{semantics} (i.e., criteria through which refine this selection) have been defined in order to establish different kinds of acceptability (see~\cite{baroni_introduction_2011} for a survey). All these semantics return two disjoint sets of arguments: ``accepted'' and ``not accepted''. The sets of accepted arguments with respect to a certain semantics are called ``extensions'' of that semantics.
An additional level of acceptability is introduced by Caminada with the reinstatement labelling, a kind of semantics that marks as rejected elements attacked by accepted arguments, and undecided the arguments that can be neither accepted nor rejected.

Dividing the arguments into just three partitions is still not sufficient when dealing with very large AFs, where one needs to limit the choice on a restricted number of selected arguments~\cite{leite_social_2011}. For this reason, a different family of semantics can be used for obtaining a broader range of acceptability levels for the arguments. Such semantics, called ranking-based, have been studied in many works, as~\cite{amgoud_ranking-based_2013,besnard_logic-based_2001,bistarelli_cooperative-game_2018,leite_social_2011,matt_game-theoretic_2008}, each focusing on a different criterion for identifying the best arguments in an AF. The idea is to assign a value to each argument through an evaluation function, so to obtain a total order over the arguments (a ranking indeed).
Although some works that map argumentation into game theory already exist, like the one in~\cite{matt_game-theoretic_2008} that introduces a new argument strength measure, to date, there is no relation between ranking-based semantics and classical ones by Dung/Caminada~\cite{dung_acceptability_1995,caminada_issue_2006}.

In this work, we present a ranking-based semantics which exploits power indexes, together with a tool that implements it. Contrary to other ranking-based semantics, ours has a strong connection with classical semantics, that are used as parameters for the evaluation of the arguments.  We rely on power indexes for establishing the ranking of the arguments since they are a very well-known concept in the fields of economics and computational social choice, where they are successfully adopted in many applications involving the fair division of costs or benefits.
This paper extends~\cite{bistarelli_tool_2019} and~\cite{bistarelli_cooperative-game_2018}, where preliminary ideas of the Shapley Value semantics and its implementation was sketched. We provide here a general and deep study on the application of power indexes to AFs and a thorough description of the implementation of our ranking-based semantics inside the ConArg suite, together with an example of how the rankings are computed by the different power indexes. We also discuss the differences among the various indices and give properties that characterise our semantics.
In Section~\ref{sec:background} we report the background on labelling and ranking-based semantics, and then the necessary preliminary notions on power indexes. Section~\ref{sec:pisem} is devoted to the definition of the ranking-based semantics we use in our tool. Afterwards, Section~\ref{sec:tool} and Section\ref{sec:example} describe the implementation of the aforementioned semantics and provide a detailed example of how the tool works, respectively. Section~\ref{sec:conclusion} wraps up the paper with conclusive thoughts and ideas about future work.

\section{Preliminaries}\label{sec:background}
In this section, we introduce the concepts from Argumentation and Game Theory we use for developing our tool. Besides the main definitions for ranking-based semantics~\cite{amgoud_ranking-based_2013}, we also provide the notion of the labelling-based semantics~\cite{caminada_issue_2006} that we implement to identify the sets of acceptable arguments.
Then, we recall the definitions of the four different power indexes used for evaluating the arguments. Those power indexes are named after the authors that formalised them and are known as the Shapley Value and the Banzhaf, Deegan-Packel and Johnston Index~\cite{keith_encyclopedia_2011}.

\subsection{Argumentation}\label{sec:arg}
An \textit{Abstract Argumentation Framework}~\cite{dung_acceptability_1995} (AF in short) consists of a set of arguments and the relations among them. Such relations, which we call ``attacks'', are interpreted as conflict conditions that allow for determining the arguments in $A$ that are acceptable together (i.e., collectively).

\begin{definition}[AF]
	An Abstract Argumentation Framework is a pair $\langle A,F \rangle$ where $A$ is a set of arguments and  $R \subseteq A \times A$ is a binary attack relation on $A$.
\end{definition}

%AFs can also be generalised by endowing the attack relations with a value representing either the ``strength'' of the attack or the probability that such attack exists in the framework. AFs of the former kind are known as \textit{weighted} AFs, while the latter is called \textit{probabilistic}.

An argumentation \textit{semantics} is a criterion that establishes which are the acceptable arguments by considering the relations among them. Two leading characterisations can be found in the literature, namely \textit{extension-based}~\cite{dung_acceptability_1995} and \textit{labelling-based}~\cite{caminada_issue_2006} semantics. While providing the same outcome in terms of accepted arguments, labelling-based semantics permit to differentiate between three levels of acceptability, by assigning labels to arguments according to the conditions stated in Definition~\ref{def:reinstatement}.

\begin{definition}[Reinstatement Labelling]\label{def:reinstatement}
	Let $F = \langle A,$ $R \rangle$ be an AF and $\mathbb{L} = \{in,out,undec\}$. A labelling of $F$ is a total function $L : A \rightarrow \mathbb{L}$. We define $in(L) = \{a \in A \mid L(a) = in \}$, $out(L) = \{a \in A \mid L(a) = out\}$ and $undec(L) = \{a \in A \mid L(a) = undec\}$.
	We say that $L$ is a reinstatement labelling if and only if it satisfies the following conditions:
	
	\begin{itemize}
		\item $\forall a,b \in A$, if $a \in in(L)$ and $(b,a) \in R$ then $b \in out(L)$;
		\item $\forall a \in A$, if $a \in out(L)$ then $\exists b \in A$ such that $b \in in(L)$ and $(b,a) \in R$.
	\end{itemize}
	
\end{definition}

The labelling obtained through the function in Definition~\ref{def:reinstatement} can be then analysed in terms of Dung's semantics~\cite{dung_acceptability_1995}.

\begin{definition}[Labelling-based semantics]\label{def:labsemnatics}
	A labelling-based semantics $\sigma$ associates with an AF $F= \langle A, R \rangle$ a subset of all the possible labellings for F, denoted as $L_\sigma(\mathit{F})$. Let $L$ be a labelling of $F$, then $L$ is
	\begin{itemize}
		\item \textbf{conflict-free} if and only if for each $a \in A$ it holds that if $a$ is labelled \emph{in} then it does not have an attacker that is labelled \emph{in}, and if $a$ is labelled \emph{out} then it has at least one attacker that is labelled \emph{in};
		\item \textbf{admissible} if and only if the attackers of each \emph{in} argument are labelled \emph{out}, and each \emph{out} argument has at least one attacker that is \emph{in};
		\item \textbf{complete} if and only if for each $a \in A$, $a$ is labelled \emph{in} if and only if all its attackers are labelled \emph{out}, and $a$ is \emph{out} if and only if it has at least one attacker labelled \emph{in};
		\item \textbf{preferred}/\textbf{grounded} if $L$ is a complete labelling where the set of arguments labelled \emph{in} is maximal/minimal (with respect to set inclusion) among all complete labellings;
		\item \textbf{stable} if and only if it is a complete labelling and $\mathit{undec}(L) = \emptyset$.
	\end{itemize}
\end{definition}

The accepted arguments of $F$, with respect to a certain semantics $\sigma$, are those labelled $in$ by $\sigma$. 
We refer to sets of arguments that are labelled \textit{in}, \textit{out} or \textit{undec} in at least one labelling of $L_\sigma(F)$ with $in(L_\sigma)$, $out(L_\sigma)$ and $undec(L_\sigma)$, respectively\footnote{We just write $L_\sigma$ when the reference to $F$ is clear and unambiguous.}.
Given an argument $a \in F$, we say that $a$ is \textit{credulously accepted} with respect to a semantics $\sigma$ if it is labelled \textit{in} in at least one extension of $\sigma$. We say that $a$ is \textit{sceptically accepted} if it is labelled \textit{in} in all extensions of $\sigma$.

In order to further discriminate among arguments, \textit{ranking-based} semantics~\cite{bonzon_comparative_2016} can be used for sorting the arguments from the most to the least preferred.

\begin{definition}[Ranking-based semantics] A ranking-based semantics associates with any $F= \langle A, R\rangle$ a ranking $\succcurlyeq_{F}$ on $A$, where $\succcurlyeq_{F}$ is a pre-order (a reflexive and transitive relation) on $A$. $a \succcurlyeq_{F} b$ means that $a$ is at least as acceptable as $b$ ($a \simeq b$ is a shortcut for  $a \succcurlyeq_{F} b$  and  $b \succcurlyeq_{F} a$, and   $a \succ_{F} b$ is a shortcut for  $a \succcurlyeq_{F} b$ and  $b \not\succcurlyeq_{F} b$).
\end{definition}

A ranking-based semantics can be characterised through some specific properties that take into account how couples of arguments in an AF are evaluated for establishing their position in the ranking. We provide a list of the properties suggested in~\cite{amgoud_ranking-based_2013} and that we use in Section~\ref{sec:example} to discuss an example of how our tool can be used for both research and applicative purposes.

\begin{definition}[Isomorphism]
	An isomorphism $\iota$ between two AFs $F=\langle A,R \rangle$ and $F'=\langle A',R' \rangle$ is a bijective function $\iota: A \rightarrow A'$ such that $\forall a,b \in A$, $(a,b) \in R$ if and only if $(\iota(a),\iota(b)) \in R'$.
\end{definition}

We can characterise the role of an argument with respect to another one according to the length of the path between them: an odd path represents an attack, while an even path is considered as a defence.

\begin{definition}[Attackers and defenders~\cite{amgoud_ranking-based_2013}]\label{def:properties} 
	Let $F = \langle A, R\rangle$ be an AF and $a,b \in A$ and denote with $P(b,a)$ a path from $b$ to $a$. The multi-sets of \textbf{defenders} and \textbf{attackers} of $a$ are $R_n^+(a) = \{b~|~\exists P(b,a)$ with length $n\in 2\mathbb{N}\}$ and $R_n^-(a) = \{b~|~\exists P(b,a)$ with length $n\in 2\mathbb{N}+1\}$, respectively. $R_1^-(a) = R^-(a)$ is the set of direct attackers of $a$.
\end{definition}

Besides arguments alone, also sets of arguments can be compared. Two rules apply: the greater the number of arguments, the more preferred the group; in case of two groups with the same size, the more preferred the arguments in a group, the more preferred the group itself.

\begin{definition}[Group comparison~\cite{amgoud_ranking-based_2013}]
	Let $\geq_S$ be a ranking on a set of arguments $A$. For any $S_1, S_2 \subseteq A$, $S_1 \geq_S S_2$ is a group comparison if and only if there exists an injective mapping $f : S_2 \rightarrow S_1$ such that $\forall a \in S_2, f(a) \succcurlyeq a$. Moreover, $S_1 >_S S_2$ is a strict group comparison if and only if $S_1 \geq_S S_2$ and $(|S_2| < |S_1|) \text{ or } \exists a \in S_2, f(a) \succ a$.
\end{definition}

Below, we list the properties proposed in~\cite{amgoud_ranking-based_2013}.

\begin{definition}[Properties of ranking-based semantics] Given a ranking-based semantics $\sigma$, an AF $F=\langle A,R \rangle$ and two arguments $a, b \in A$, the following properties are defined.
	
	\textit{Abstraction}. For any isomorphism $\iota$ such that $F' = \iota(F)$, $a \succcurlyeq_{F}^\sigma b$ if and only if $\iota(a) \succcurlyeq_{F'}^\sigma \iota(b)$.
	
	\textit{Independence}. Let $cc(F)$ be the set of connected components in $F$. $\forall F' = \langle A',R' \rangle \in cc(F),$ $\forall a,b \in A'$, then $a\succcurlyeq_{F'}^\sigma b \Rightarrow a \succcurlyeq_{F}^\sigma b$.
	
	%\textit{Void Precedence.} $R_1^-(a) = \emptyset$ and $R_1^-(b) \neq \emptyset \Rightarrow a \succ^\sigma b$.
	
	\textit{Self-contradiction}. $(a,a) \not\in R$ and $(b,b) \in R \Rightarrow a \succ^\sigma b$.
	
	\textit{Cardinality Precedence}. $|R_1^-(a)| < |R_1^-(b)| \Rightarrow a \succ^\sigma b$.
	
	\textit{Quality Precedence}. $\exists c \in R_1^-(b)$ such that $\forall d \in R_1^-(a)$, $c \succ^\sigma d \Rightarrow a \succ^\sigma b$.
	
	%\textit{Counter-Transitivity}. $R_1^-(b) \geq_S R_1^-(a) \Rightarrow a \succcurlyeq^\sigma b$.
	
	%\textit{Strict Counter-Transitivity}. $R_1^-(b) >_S R_1^-(a) \Rightarrow a \succ^\sigma b$.
	
	%\textit{Defense Precedence}. $|R_1^-(a)| = |R_1^-(b)|$, $R_2^+(a) \neq \emptyset$ and $R_2^+(b) = \emptyset \Rightarrow a \succ^\sigma b$.
	
	\textit{Non-attacked Equivalence}. $R^{-}(a)= \emptyset$ and $R^{-}(b) = \emptyset \Rightarrow a \simeq^\sigma b$.
	
	\textit{Totality}. $a \succcurlyeq^\sigma b$ or $b \succcurlyeq^\sigma a$.
	
\end{definition}

The ranking-based semantics we present has been implemented in ConArg\footnote{ConArg website: http://www.dmi.unipg.it/conarg/}, a web tool that implements various work we conduct in the field of Abstract Argumentation. The core component of the whole suite is the computational framework~\cite{bistarelli_conarg_2011}, based on Constraint Programming, that is able to solve different problems related to AFs. ConArg can be used for many purposes, as computing semantics, visualising Argumentation Frameworks (AF) together with the computed extensions, programming user application using a predefined AF library, and studying properties of semantics and AFs. Additional modules allow for dealing with weighted~\cite{bistarelli_conarg_2016} and probabilistic~\cite{bistarelli_probabilistic_2018} AFs.
Our semantics evaluates the arguments of an AF by using the notion of power index, that we describe in the following section.

\subsection{Power Indexes}\label{sec:voting}
In game theory, cooperative games are a class of games where groups of players (or agents) are competing to maximise their goal, through one or more specific rules. Voting games are a particular category of cooperative games in which the profit of coalitions is determined by the contribution of each player.
To identify the ``value'' brought from a single player to a coalition, power indexes are used to define a preference relation between different agents, computed on all the possible coalitions.
In our work, we use four among the most commonly used power indexes, namely the Shapley Value~\cite{shapley_contributions_1953,WINTER20022025}, the Banzhaf Index~\cite{banzhaf_weighted_1965}, the Deegan-Packel Index and the Johnston Index~\cite{keith_encyclopedia_2011}.

Every power index relies on a characteristic function $v: 2^{N}\rightarrow \mathbb{R}$ that, given the set $N$ of players, associates each coalition $S \subseteq N$ with a real number in such a way that $v(S)$ describes the total gain that agents in $S$ can obtain by cooperating with each other.
The expected marginal contribution of a player $i \in N$, given by the difference of gain between $S$ and $S \cup\{i\}$, is $v_{S_i} = v(S \cup \{i\})  - v(S)$.

The Shapley Value $\phi_i(v)$ of the player $i$, given a characteristic function $v$, is computed as:
\begin{equation} \label{eq:sv}
\phi_i(v)=  \sum_{S \subseteq N \setminus \{i\} }\frac{|S|! \; (|N| - |S| - 1)! }{|N|!} \;\; v_{S_i}
\end{equation}

The formula considers a random ordering of the agents, picked uniformly from the set of all $|N|!$ possible orderings. The value $|S|! \; (|N| - |S| - 1)!$ expresses the probability that all the agents in $S$ come before $i$ in a random ordering.

The second fair division scheme we use is the Banzhaf Index $\beta_i(v)$, which evaluates each player $i$ by using the notion of \textit{critical voter}: given a coalition $S \subseteq N \setminus \{i\}$, a critical voter for $S$ is a player $i$ such that $S \cup \{i\}$ is a winning coalition, while $S$ alone is not. In other words, $i$ is a critical voter if it can change the outcome of the coalition it joins.

\begin{equation} \label{eq:ban}
\beta_i(v) = \frac{1}{2^{|N|-1}} \sum_{S \subseteq N \setminus \{i\}  }  v_{S_i}
\end{equation}

The difference between the Shapley Value and the Banzhaf index is that the latter does not take into account the order in which the players form the coalitions.

Deegan and Packel assume that only minimal winning coalitions are formed, that they do so with equal probability, and that if such a coalition is formed it divides the (fixed) spoils of victory equally among its members. In order to avoid divisions by zero in the formula, we use the interpretation of~\cite{alonso-meijide_power_2015}: let's call $M(v)$ the set of minimal winning coalitions of the game (always assuming $\emptyset \in M(v)$), and $M_i(v)$ the subset of $M(v)$ formed by coalitions $S \subseteq N$ such that $i \in S$. The Deegan-Packel index $\rho_i(v)$ of a player $i \in N$ is computed as follows. 
\begin{equation} \label{eq:dp}
\rho_i(v) = \frac{1}{|M(v)|} \sum_{\substack{S \subseteq M_i(v) \setminus \{i\} \\ S\neq \emptyset}} \frac{v_{S_i}}{|S|}
\end{equation}

The last index we implement is the Johnston index~\cite{duran_computing_2003}. Based on the principle of critical vote, it differs from Banzhaf's for the fact that critical voters in winning coalitions are rewarded with a fractional score instead of one whole unit (that is the score is equally divided among all critical members of the coalition).
Let $\varkappa(S)$ denote the number of critical voters in a winning coalition $S$. The Johnston index $\gamma_i(v)$ of a player $i \in N$ is computed as follows.
\begin{equation} \label{eq:jo}
\gamma_i(v) = \sum_{\substack{S \subseteq N \setminus \{i\} \\ \varkappa(S) \geq 1}} \frac{v_{S_i}}{\varkappa(S)}
\end{equation}

Notice that the summation of the Equation~\ref{eq:jo} is only done over the coalitions in which there is at least one critical voter.
In the following section, we give the definition of our ranking-semantics based on power indexes, that we call ``PI-based semantics''.

\section{PI-Based Semantics}\label{sec:pisem}
In order to rank arguments of a framework through the use of a power index, we need, first of all, to define the characteristic function that evaluates the coalition formed by the arguments.

\begin{definition}[Characteristic function]\label{def:charfun} Let $F = \langle A,R \rangle $ be an AF, $\sigma$ a Dung semantics and $L_\sigma$ the set of all possible labellings on $F$ satisfying $\sigma$. For any $S\subseteq A$, the labelling-based characteristic functions $v^{I}_{\sigma,F}$ and $v^{O}_{\sigma, F}$ are defined as:
	
	\begin{align*}
	&v^{I/O}_{\sigma,F}(S) =
	\begin{cases}
	1, & \text{if } S \in \mathit{in/out}(L_\sigma) \\
	0, & \text{if } otherwise
	\end{cases}
	\end{align*}
	
\end{definition}

The function $v^{I}_{\sigma,F}(S)$ takes into account the acceptability of a set of arguments $S$ with respect to a certain semantics $\sigma$, assigning to such set a score equal to $1$ if there exists a labelling $L_\sigma$ in which all and only the arguments of $S$ are labelled \textit{in}. In other words, a set is positively evaluated by $v^{I}_{\sigma,F}$ only if it represents an extension for the semantics $\sigma$, and the higher the score of the power index, the better the rank of an argument. A second characteristic function, $v^{O}_{\sigma, F}(S)$, is then introduced to break possible ties in the final ranking. In the case two arguments of F have the same power index with respect to the function $v^{I}_{\sigma,F}$, we compare the evaluations obtained through $v^{O}_{\sigma, F}$, that considers the sets of arguments labelled \textit{out} by $\sigma$. This further evaluation has a negative interpretation: the higher the score according to $v^{O}_{\sigma, F}$, the worse the rank.

\begin{definition}[PI-based semantics]
	Let $F = \langle A,R \rangle$ be an AF, $\sigma$ a Dung semantics, $\pi \in \Pi:$\{$\phi,\beta,\rho,\gamma$\} a power index, and $v^{I}_{\sigma,F}$, $v^{O}_{\sigma,F}$ the characteristic functions.
	The PI-based semantics associates to $F$ a ranking $\succcurlyeq_{F}^{\pi}$ on $A$, defining a lexicographic order on the pairs $(v^{I}_{\sigma,F}, v^{O}_{\sigma,F})$ such that $\forall a,b \in A$, $a \succ_{F}^{\pi} b$ if and only if
	\begin{itemize}
		\item $\pi_a(v^I_{\sigma, F}) > \pi_b(v^I_{\sigma, F})$, or
		\item $\pi_a(v^I_{\sigma, F}) = \pi_b(v^I_{\sigma, F})$ and $\pi_a(v^O_{\sigma, F}) < \pi_b(v^O_{\sigma, F})$
	\end{itemize}
	and that $a \simeq_{F}^{\pi} b$ if and only if
	\begin{itemize}
		\item$\pi_a(v^I_{\sigma, F}) = \pi_b(v^I_{\sigma, F})$ and $\pi_a(v^O_{\sigma, F}) = \pi_b(v^O_{\sigma, F})$.
	\end{itemize}
\end{definition}

We study which properties among those in Definition~\ref{def:properties} are satisfied by the PI-semantics obtained through the Shapley Value, and which are not.

\begin{theorem}\label{th:prop_sat}
	Consider an AF $F=\langle A,R\rangle$, two arguments $a,b \in A$, a Dung semantics $\sigma$ and the power index $\phi$. The PI-based semantics satisfies the following properties:
	\begin{itemize}
		\item \textit{Abstraction}, \textit{Independence} and \textit{Totality} for any $\sigma \in \{$conflict-free, admissible, complete, preferred, stable$\}$
		\item \textit{Self-contradiction} only for $\sigma$ = conflict-free
		\item \textit{Non-attacked Equivalence} only for $\sigma \in \{$complete, preferred, stable$\}$
	\end{itemize}
	
	For any $\sigma \in \{$conflict-free, admissible, complete, preferred, stable$\}$, the PI-based semantics does not satisfy \textit{Cardinality Precedence} and \textit{Quality Precedence}.
\end{theorem}

\begin{proof}
	For each power index, characteristic function and semantics, we state if the properties are satisfied.
	\begin{itemize}
		
		\item \textit{Abstraction}. Any extension of every semantics $\sigma$ is computed starting from the set of attack relations among arguments, thus the ranking is preserved up to isomorphisms of the framework.
		
		\item \textit{Independence}. The semantics we propose computes the ranking starting from the sets of extensions of a chosen semantics $\sigma$. Since the labelling of each argument $a$ is determined by the other arguments in the same connected component of $a$, also the ranking between every pair of arguments $a$ and $b$ is independent of any other argument outside the connected component of $a$ and $b$.
		
		\item \textit{Self-contradiction}. Consider $\sigma$ = conflict-free. If $(a,a) \notin R$ and $(b,b) \in R$, we can state that
		\begin{center}
			$\exists E \subset A \wedge a \notin E: v^I_{\sigma,F}(E \cup \{a\}) - v^I_{\sigma,F}(E) > -1 ~\wedge$\\
			$\nexists E \subset A \wedge b \notin E: v^I_{\sigma,F}(E \cup \{b\}) - v^I_{\sigma,F}(E) > -1$
		\end{center}
		Thus $\phi_a(v^I_{\sigma,F}) > \phi_b(v^I_{\sigma,F})$ from which we conclude that $a \succ_{\sigma,F}^{\pi} b$ when $\sigma$ = conflict-free.
		In Figure~\ref{fig:ce3} we show a counterexample for the other cases.
		
		\item \textit{Cardinality Precedence}. The property is not satisfied for any $\sigma \in$ \{conflict-free, admissible, complete, preferred, stable\} with respect to $\phi$. Counterexample in Figure~\ref{fig:ce4}.
		
		\item \textit{Quality Precedence}. The property is not satisfied for any $\sigma \in$ \{conflict-free, admissible, complete, preferred, stable\} with respect to $\phi$. See Figures~\ref{fig:ce5},~\ref{fig:ce6} and~\ref{fig:ce7} for counterexamples.
		
		\item \textit{Non-attacked Equivalence}. Non-attacked arguments are labelled \textit{in} in every complete extension, thus, if two arguments $a,b \in A$ are non-attacked, then we have $\pi_a(v^I_{\sigma,F}) = \pi_b(v^I_{\sigma,F})$ and $\pi_a(v^O_{\sigma,F}) = \pi_b(v^O_{\sigma,F})$. Hence $a \simeq_{F}^{\pi} b$ when $\sigma$ = complete. Since all the preferred and stable extensions are also complete, \textit{Non-attacked Equivalence} holds for $\sigma \in \{$complete, preferred, stable$\}$. On the other hand, for $\sigma \in \{\text{conflict-free, admissible}\}$ the property is not satisfied (see the counterexample in Figure~\ref{fig:ce8}).
		%finire
		
		\item \textit{Totality}. The Shapley Value associates a real number to every arguments of an AF, thus all pairs of arguments can be compared through the order of $\mathbb{R}$.
		
\end{itemize}\end{proof}

\begin{figure}[htb]
	%	arg(a).
	%	arg(b).
	%	arg(c).
	%	att(c,c).
	%	att(a,b).
	\centering
	\includegraphics[height=1.1cm]{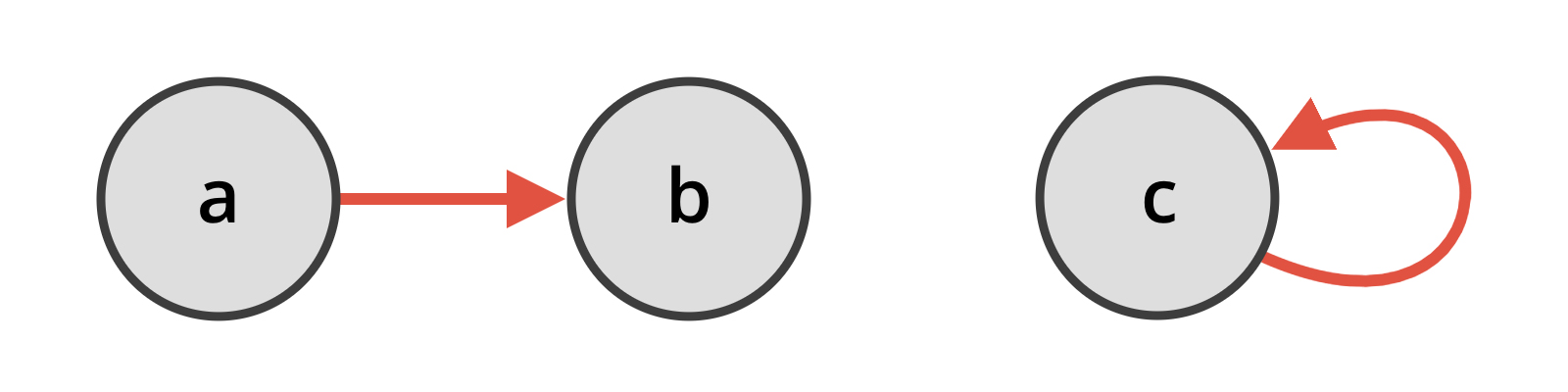}
	\caption{Counterexample for \textit{Self-contradiction}: we have $c \succ_{F}^{\phi} b$ when $\sigma \in \{$admissible, complete, preferred$\}$, and $b \simeq_{F}^{\phi} c$ when $\sigma$ = stable.}
	\label{fig:ce3}
\end{figure}

\begin{figure}[htb]
	%	arg(a).
	%	arg(b).
	%	arg(c).
	%	arg(d).
	%	arg(e).
	%	arg(f).
	%	att(a,b).
	%	att(b,c).
	%	att(b,d).
	%	arg(g).
	%	att(e,g).
	%	att(f,g).
	\centering
	\includegraphics[height=2.2cm]{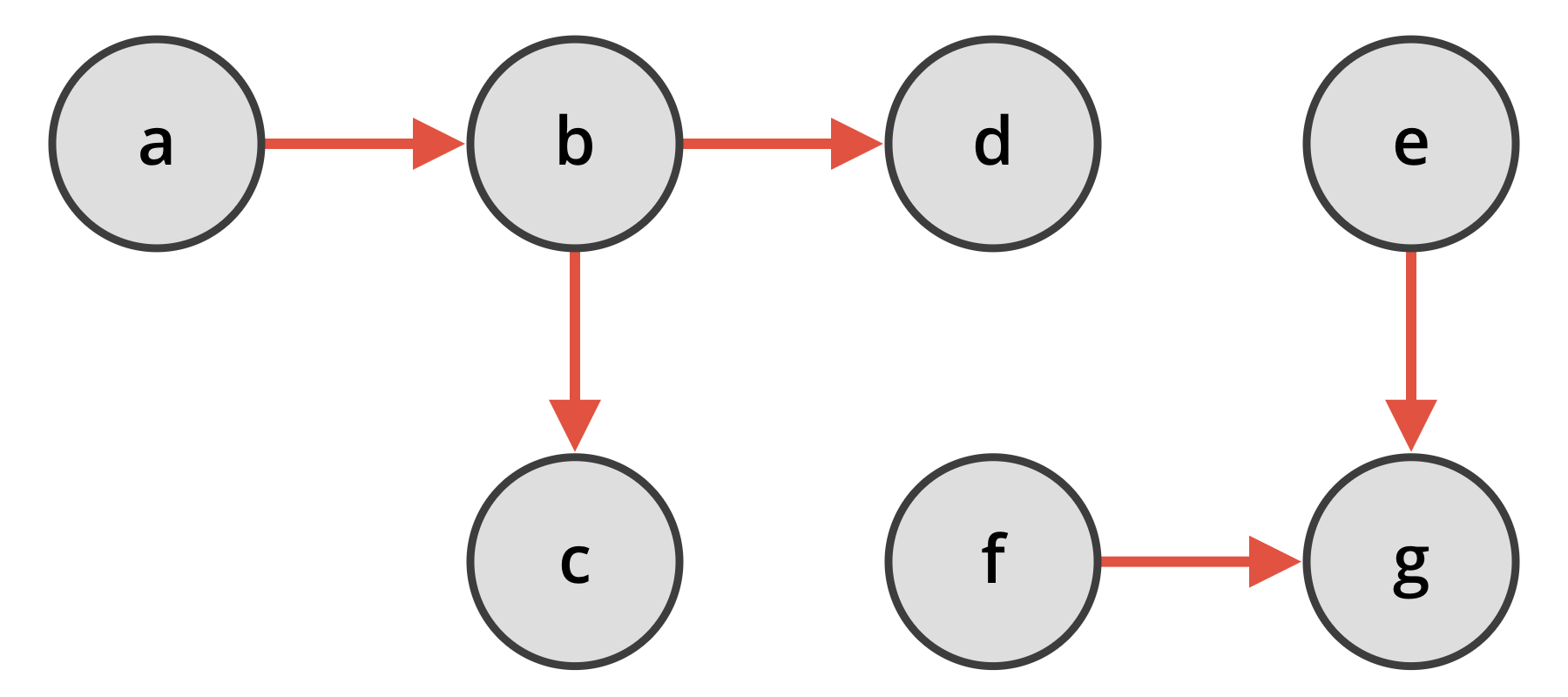}
	\caption{A counterexample for \textit{Cardinality Precedence}. The argument $g$ has more direct attackers than $b$. However, $g \succ_{F}^{\phi} b$ when $\sigma$ = conflict-free, while $b \simeq_{F}^{\pi} g$ when $\sigma \in \{$admissible, complete, preferred, stable$\}$.}
	\label{fig:ce4}
\end{figure}

\begin{figure}[htb]
	%	arg(a).
	%	arg(b).
	%	arg(c).
	%	arg(d).
	%	arg(e).
	%	arg(f).
	%	att(a,b).
	%	att(c,d).
	%	att(c,e).
	%	att(d,f).
	\centering
	\includegraphics[height=2.2cm]{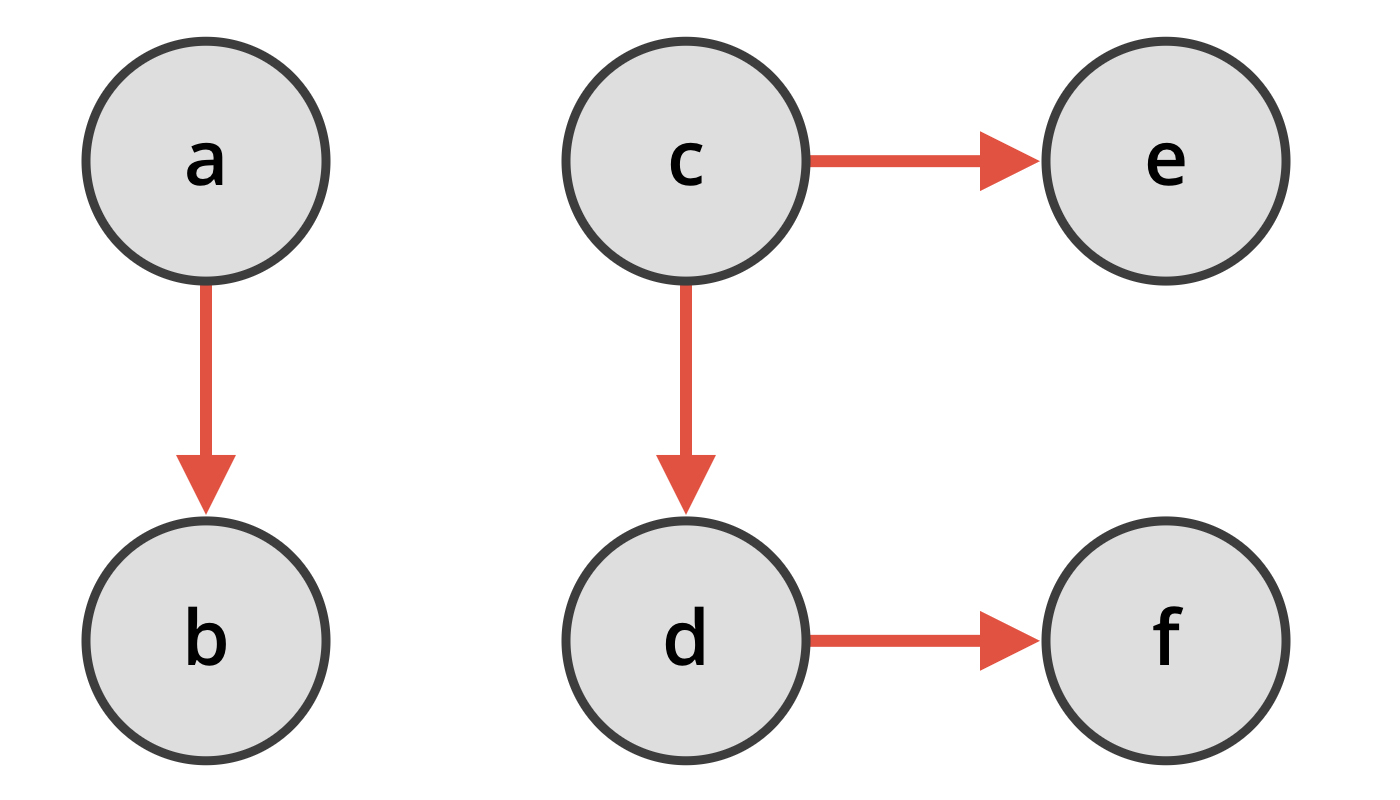}
	\caption{A counterexample for \textit{Quality Precedence} of PI-based semantics: when $\sigma$ = conflict-free, $a \succ_{F}^{\phi}  c$ and  $b \succ_{F}^{\phi}  d$. If $\sigma$ = admissible, we have instead $c \succ_{F}^{\phi}  a$ and  $d \succ_{F}^{\phi}  b$.}
	\label{fig:ce5}
\end{figure}

\begin{figure}[htb]
	%	arg(a).
	%	arg(b).
	%	arg(c).
	%	arg(d).
	%	att(a,b).
	%	att(a,d).
	%	att(b,a).
	%	att(b,c).
	%	att(a,c).
	%	att(c,d).
	\centering
	\includegraphics[height=2.2cm]{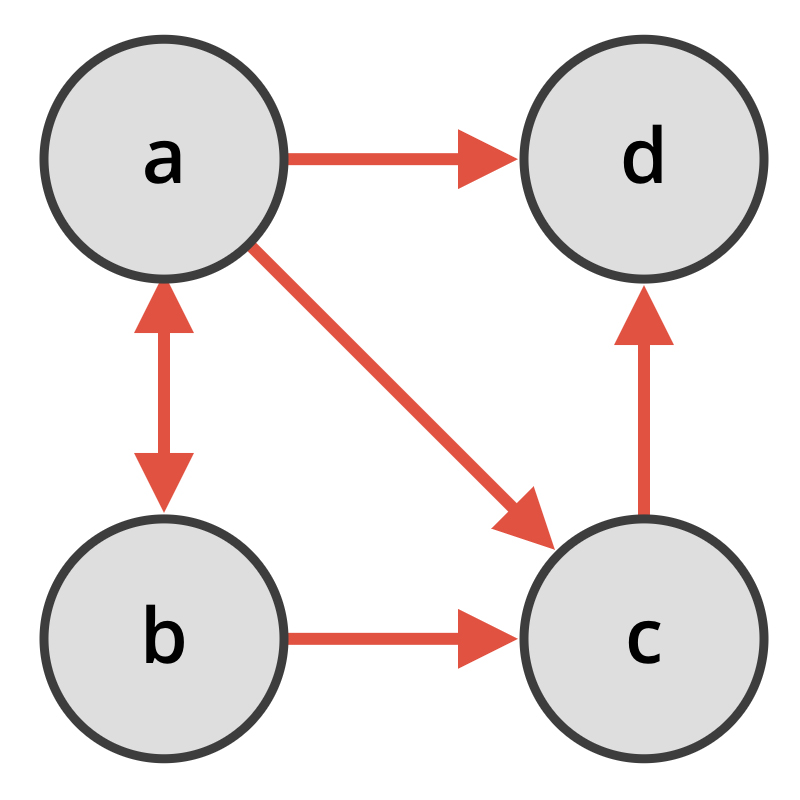}
	\caption{A counterexample for property \textit{Quality Precedence}. If $\sigma$ = complete, we have $b \succ_{F}^{\phi} c$ and  $a \succ_{F}^{\phi} d$.}
	\label{fig:ce6}
\end{figure}

\begin{figure}[htb]
	%	arg(a).
	%	arg(b).
	%	arg(c).
	%	arg(d).
	%	arg(e).
	%	arg(f).
	%	att(a,b).
	%	att(c,d).
	%	att(d,c).
	%	att(c,f).
	%	att(f,e).
	%	att(c,e).
	%	att(d,e).
	\centering
	\includegraphics[height=2.2cm]{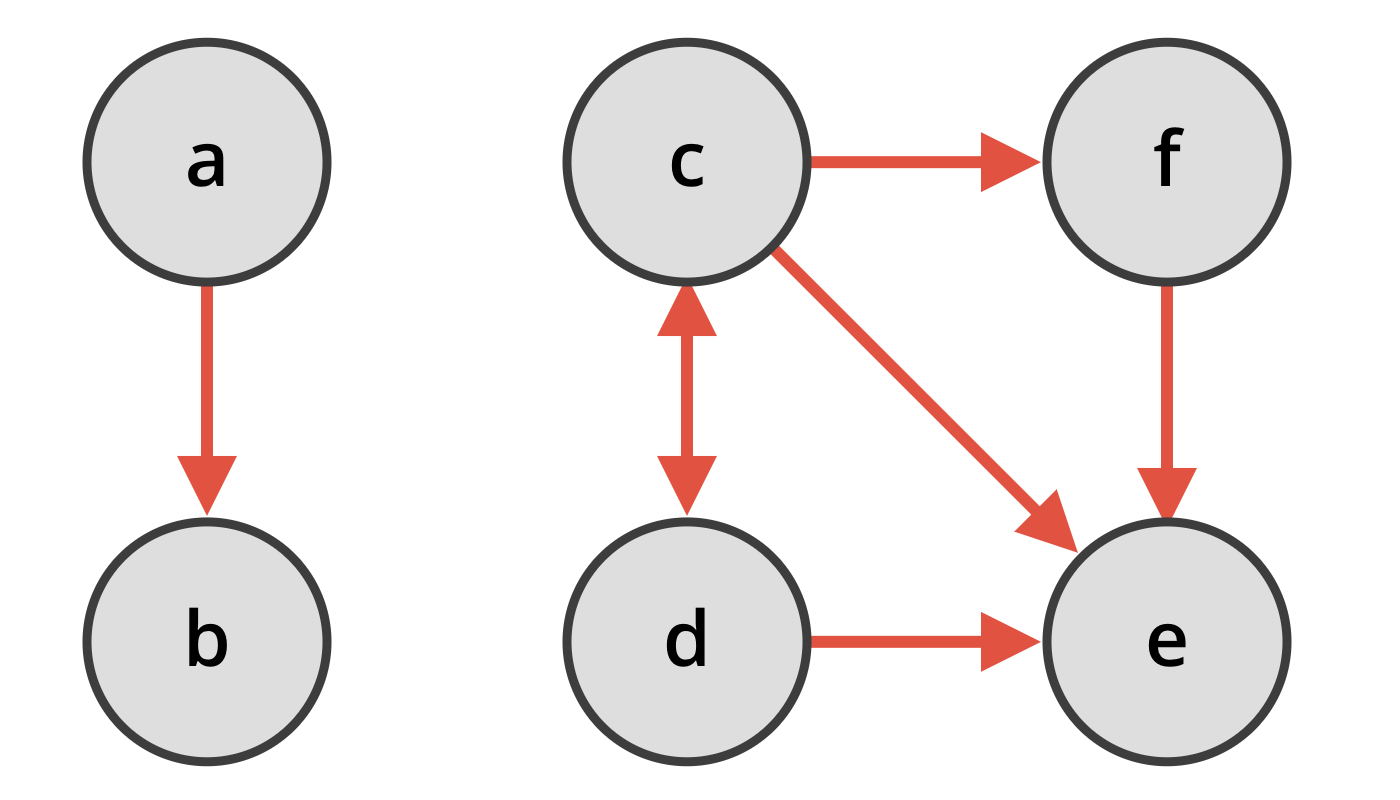}
	\caption{A counterexample for property \textit{Quality Precedence} of PI-based semantics. When $\sigma \in \{\text{preferred, stable}\}$, $a \succ_{F}^{\phi} c$ and  $b \simeq_{F}^{\phi} e$ holds.}
	\label{fig:ce7}
\end{figure}

\begin{figure}[htb]
	%	arg(a).
	%	arg(b).
	%	arg(c).
	%	arg(d).
	%	att(a,b).
	%	att(b,c).
	\centering
	\includegraphics[height=1.1cm]{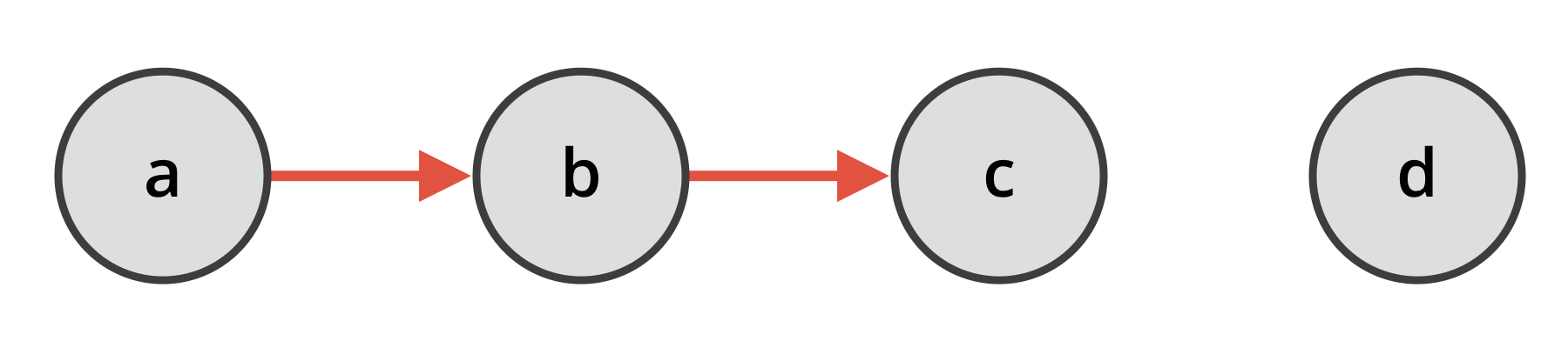}
	\caption{Counterexample for property \textit{Non-attacked Equivalence} of PI-based semantics. For $\sigma$ = conflict-free, $d \succ_{F}^{\phi} a$, and when $\sigma$ = admissible, $a \succ_{F}^{\phi} d$.}
	\label{fig:ce8}
\end{figure}

\textit{Abstraction}, \textit{Independence} and \textit{Totality} are desirable properties, since they guarantee that a total order can always be established over the arguments of an AF, only considering the structure of the underlying graph and the relations among the arguments.
The \textit{Self-contradiction} property ensures, for the conflict-free semantics, that self-attacking arguments have a lower ranking than the others. Indeed, the conflict-free semantics only takes into account whether there are attacks among the arguments. For the other semantics, it may happen that an argument attacked by another argument with a high value is ranked lower than a self-attacking argument: in other words, when the notion of defence is taken into account, an argument which is defeated by a solid counterargument has less value than a contradictory argument.
For the complete, preferred and stable semantics, which always label \textit{in} the arguments that do not receive attacks, the \textit{Non-attacked Equivalence} property allows for knowing the value of all the non-attacked arguments just by computing the value for one of them.
\textit{Cardinality Precedence} and \textit{Quality Precedence} never hold. In fact, the ranking of an argument $a$ does not only depend on either the number of attackers of $a$ or their position in the ranking, but also on how many other arguments are defended by $a$. We plan to study such property when extending our work to ranking-based semantics over weighted AFs.

Given a ranking, we can correlate the value given to each argument to its credulous/sceptical acceptance. Looking at the acceptability of an argument, we can have in advance some information about the value of its evaluation, without even computing the power index.

\begin{theorem}\label{th:skcr}
	Let $F = \langle A,R \rangle$ be an AF, $\pi \in \Pi:$\{$\phi,\beta,\rho,\gamma$\} a power index, and $v^{I}_{\sigma,F}$ a characteristic function. Then
	\begin{itemize}
		\item if $a$ is sceptically accepted $\implies$ $\pi_a(v^I_{\sigma, F}) > 0$;
		\item if $a$ is credulously rejected $\implies$ $\pi_a(v^I_{\sigma, F}) < 0$;
	\end{itemize}
\end{theorem}

\begin{proof}
	The proof is straightforward and can be derived from the definition of the power indexes.
\end{proof}

Sceptically accepted arguments are ranked higher than credulously accepted and rejected arguments. Analogously, credulously accepted arguments are ranked higher that rejected arguments.

\begin{definition}\label{def:new_prop}
	Let $F = \langle A,R \rangle$ be an AF, $a,b \in F$ two arguments, $\sigma$ a ranking-based semantics and $\delta$ a Dung's semantics. We define the following properties.
	
	\textbf{Sceptical Precedence ($\delta$-ScP).}  If $a$ is sceptically accepted with respect to $\delta$ and $b$ is not, than $a \succ^\sigma b$.
	
	\textbf{Credulous Precedence ($\delta$-CrP).}  If $a$ is credulously accepted with respect to $\delta$ and $b$ is always rejected, than $a \succ^\sigma b$.
\end{definition}

The following proposition holds.

\begin{proposition}
	The PI-based semantics satisfies $\delta$-ScP and $\delta$-CrP for any $\delta \in \{$conflict-free, admissible, complete, preferred, stable$\}$.
\end{proposition}

\begin{proof}
	Given an AF $F$, a credulously accepted argument $i$ with respect to $\sigma$ and an evaluation function $v^I_{\sigma,F}$, there exists at least one subset of arguments $S$ such that $v^I_{\sigma,F}(S_{-i} \cup \{i\}) - v^I_{\sigma,F}(S_{-i}) > 0$. Thus the value of $i$ will always be higher than that of any rejected argument $j$, for which $v^I_{\sigma,F}(S_{-j} \cup \{j\}) - v^I_{\sigma,F}(S_{-j}) < 0$ for any $S$, and $\sigma$-CrP holds. We can make the same consideration for $\sigma$-SkP, showing that skeptically accepted arguments have higher value than the others.
\end{proof}

As a first step for comparing our semantics with other from the literature, we have checked if property $\delta$-CrP is satisfied by some of the ranking-based semantics surveyed in~\cite{bonzon_comparative_2016} that return a total ranking, namely \textit{Cat}, \textit{Dbs}, \textit{Bds}.

\begin{proposition}
	The ranking-based semantics \textit{Cat}, \textit{Dbs} and \textit{Bds} do not satisfy the $\delta$-CrP property when $\delta \in$ \{admissible, complete, preferred, stable\}.
	%	Categoriser (Cat)
	%	Discussion-based semantics (Dbs)
	%	Burden-based semantics (Bds
\end{proposition}

Indeed, considering the example in Figure~\ref{fig:ex-cat}, we have that argument $d$ is always preferred to $e$ in the rankings obtained by using Cat, Dbs and Bds.

\begin{figure}[htb]
	\centering
	\includegraphics[height=2.2cm]{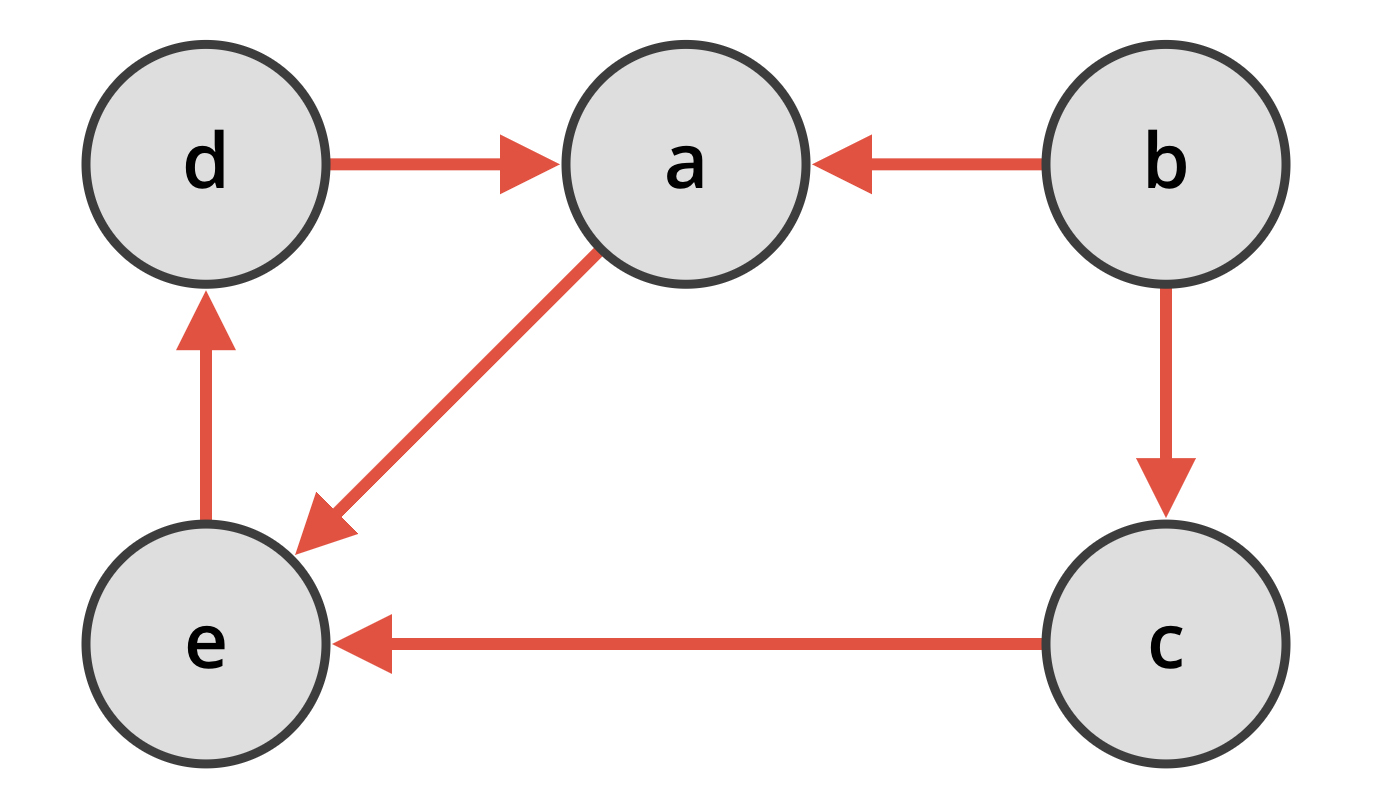}
	\caption{Example of an AF where $d \succ^\sigma e$ for $\sigma \in$ \{Cat, Dbs, Bds\}.}
	\label{fig:ex-cat}
\end{figure}

The graded semantics proposed in~\cite{grossi_graded_2015} also takes into account extensions of classical semantics in order to determine an ordering between arguments of an AF. The two principles on which the semantics is based are: having fewer attackers is better than having more; having more defenders is better than having fewer. Since the authors assume the sceptical definition for the justification of the arguments, the graded semantics satisfies both $\delta$-ScP and $\delta$-CrP. However, the used order relation is only partial (and thus some of the arguments may be incomparable). Moreover, the ranking being built on the two principles mentioned above does not allow to catch the real contribution of the arguments in forming the extensions, that, instead, is the intention of the PI-based semantics.

Finally, we discuss the ranking semantics, based on subgraphs analysis, introduced in~\cite{dondio_ranking_2018}. This semantics sorts the arguments of an AF by establishing a lexicographical ordering between the values of a tuple that contains, for each argument $a$, the label assigned to $a$ by a certain semantics, and the number of times $a$ is labelled $l$ over the total number of subgraphs, for $l=$ \textit{in}, \textit{out} and \textit{undec}, respectively. The semantics satisfies $\delta$-ScP and $\delta$-CrP. The main difference with our approach is that, while we only consider acceptable extensions for obtaining the evaluation of an argument, the semantics in~\cite{dondio_ranking_2018} uses all the possible subsets of arguments for computing the ranking, leading to results that do not fit the definition of the chosen Dung semantics.
For instance, the ranking returned by the PI-based semantics for the AF in Figure~\ref{fig:ex1}, with respect to the preferred semantics, is $a \simeq c \succ b$, since the only preferred extension is $\{a,c\}$.
Similar considerations also hold for the example in Figure~\ref{fig:ex2} for the preferred semantics, where the PI-based semantics returns $a \simeq b \simeq c$ for the functions $\phi$ and $\beta$, and $b \succ a \simeq b$ for $\gamma$. Those rankings reflects the meaning of preferred semantics: arguments $a$ and $c$ are both necessary for obtaining a preferred extension, so they are ranked the same. According to the Johnston Index $\gamma$, argument $b$ is the best one since it is the only critical voter of its coalition (i.e., the extension $\{b\}$). On the other hand, for the admissible semantics, $\phi$ and $\beta$ return the ranking $c \succ a \simeq b$. In fact, $c$, that is admissible alone, also reinstate $a$. The tie between $a$ and $b$ is broken when considering the function $\gamma$, for which $c \succ a \succ b$.

\begin{figure}[htp]
	\centering
	\subfloat[Reinstatement and rebuttal attack. Admissible and preferred semantics are $\mathit{ADM}$ = $\{ \emptyset,$ $\{a\},$ $\{c\},$ $\{a, c\} \}$ and $\mathit{PRE} = \{ \{a, c\} \}$, respectively.]{%
		\includegraphics[width=\linewidth]{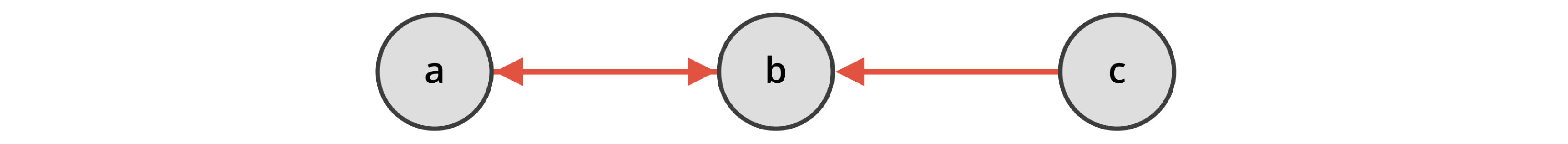}%
		\label{fig:ex1}%
	}
	
	\subfloat[Reinstatement with rebuttal attack. Admissible and preferred semantics are $\mathit{ADM}$ = $\{ \emptyset,$ $\{b\},$ $\{c\},$ $\{a, c\} \}$ and $\mathit{PRE} = \{ \{b\}, \{a, c\} \}$, respectively.]{%
		\includegraphics[width=\linewidth]{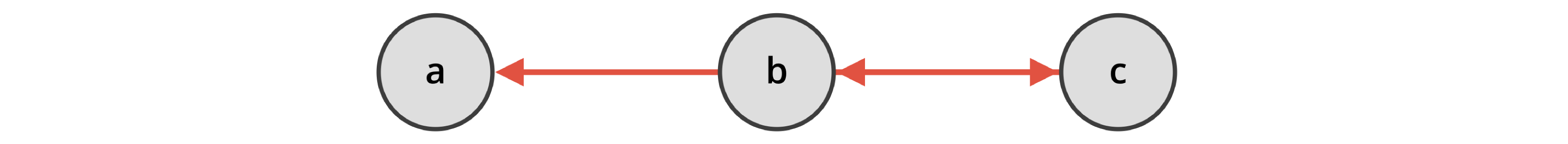}%
		\label{fig:ex2}%
	}
	
	\caption{Example of two AFs with different rebuttal/restatement configurations.}
	
\end{figure}

In the following section, we show how the PI-based ranking semantics have been implemented in ConArg.

\section{Tool Description}\label{sec:tool}
%Recently, in~\cite{bistarelli_tool_2019}, we add to ConArg the capability of handling PI-based ranking semantics, implementing the Shapley Value and the Banzhaf Index. In this paper, we extend the support for PI-based semantics by also introducing the other power indexes discussed in Section~\ref{sec:voting}, namely the Deegan-Packel and the Johnston indexes.

The ConArg Web Interface (see Figure~\ref{fig:interface} for an overview) allows one to easily perform complex argumentation related tasks. Below, we describe the main features of the tool, highlighting those introduced more recently. 

\begin{figure*}[t]
	\centering
	\includegraphics[width=0.8\linewidth]{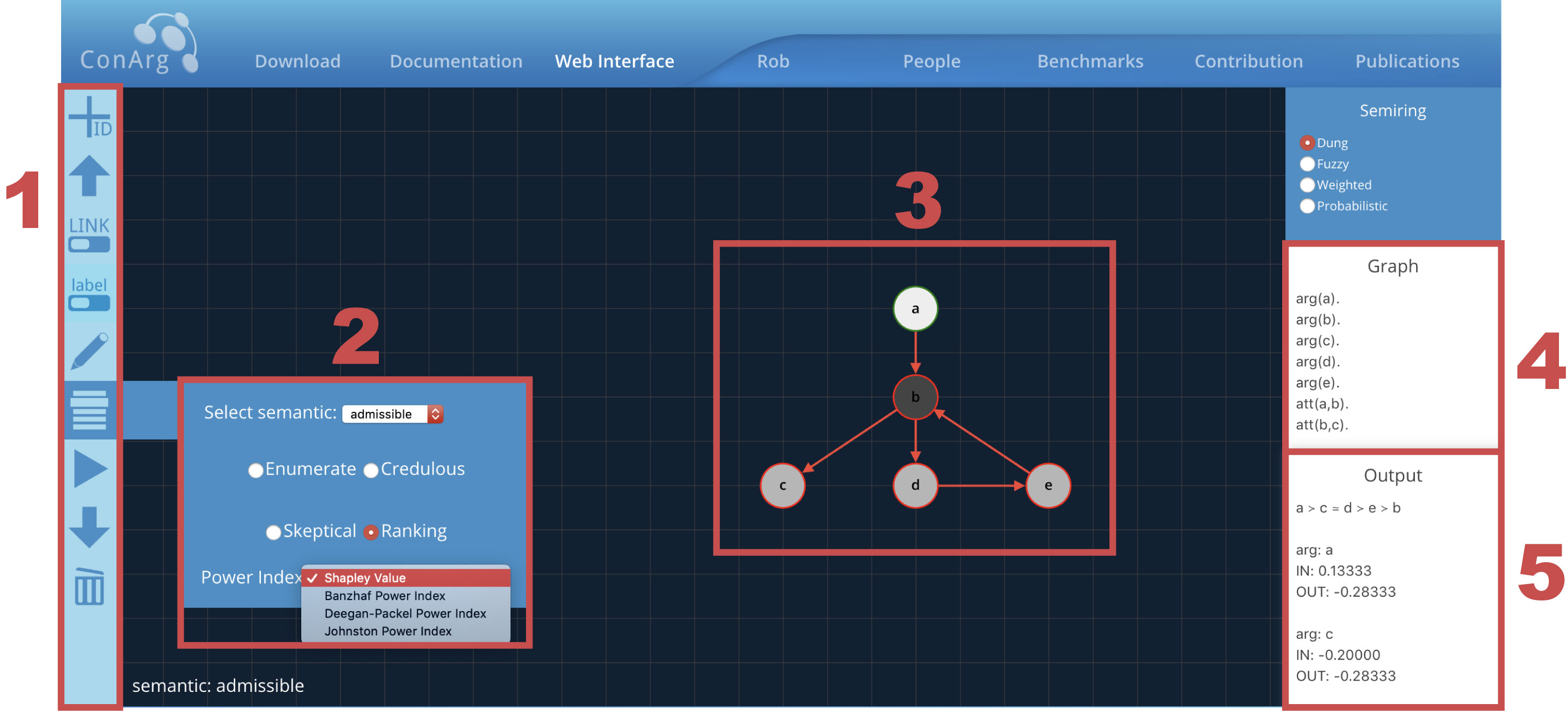}
	\caption{A screenshot of the ConArg web interface. The highlighted elements are: 1. Options menu, 2. Semantics selection panel, 3. Canvas where the AF is visualised, 4. AF in input, 5. Output panel.}
	\label{fig:interface}
\end{figure*}

\textbf{Menu.} Positioned to the left side of the interface, it allows for choosing among different options for both visualising AFs and solving argumentation problems.% In particular, it is possible to: import an AF from a local file in aspartix format, change the visualisation of the attacks for weighted AFs, switching the behaviour of the drag action between moving an argument in the canvas and drawing an attack, selecting the kind of problem to solve, run the computation, and save the output as a text file.

\textbf{Semantics selection panel.} Here it is possible to set the parameters for the resolution of several problems. First of all, one is required to select a Dung semantics through the dedicated drop-down menu. For each semantics, four different kinds of problem can be solved.% In particular, one can: enumerate the extensions for the chosen semantics, check the credulous/sceptical acceptability of a particular argument, and rank the arguments by using a PI-based semantics. For the latter problem, it is required to select a power index among the four that are implemented.

\textbf{Canvas.} This area of the interface has a twofold purpose. On one hand, it is possible to define an AF by drawing nodes and edges. On the other hand, after the calculation of a solution for a certain problem, the canvas allows for visualising the output directly on the displayed AF, through a specific colouration of the arguments.% For instance, to display the results of a ranking-based semantics, arguments are assigned a greyscale colour according to their ranking position.

%\textbf{Semiring selection panel.} Different paradigms of representation for the AFs can be selected, in addition to the classical one. For each semiring, different options on the menu are available.

\textbf{AF in input.} AFs can be entered in this panel. Changes to the canvas also affect this area, that maintains a coherent representation of the AF.

\textbf{Output panel.} The solutions for the various problems solved by ConArg are displayed here. It is also possible to download a text file containing the output.

The tool relies on JavaScript and PHP classes to implement the various components.
The graph drawing functions are provided by \textit{D3.js}, a JavaScript library for visualising and manipulating data (see~\cite{bistarelli_tool_2019} for details).

\subsection{Implementation of the PI-Based Semantics}
Behind the web interface, ConArg has several modules (like the solver and the ranking script) that allow one to access different functionalities to cope with argumentation problems~\cite{bistarelli_conarg_2011,bistarelli_conarg_2016}. A library containing the ConArg source code is also available online\footnote{http://www.dmi.unipg.it/conarg/download.html}. In this section, we discuss, in particular, the component of the tool that concerns ranking-based semantics, putting attention on implementation aspects.

When we start the computation of the ranking over the arguments of an AF $F$, the interface calls the ConArg solver that returns the set $S$ of extensions for the chosen Dung semantics $\sigma$. These extensions represent the sets of \textit{in} arguments with respect to $\sigma$ and are formatted as sets of strings (e.g., $S = \{\{a\},\{a,b\},\{c,d\}\}$, where $a,b,c$ and $d$ are arguments). Together with the set of extensions, also the framework $F$ and the power index $\pi$ that we want to use are passed to the ranking script. The script, then, computes the specified power index $\pi$ for each of the argument in $F$. The obtained values are approximated to the nearest fifth decimal digit.
The four functions that implement the equations of Section~\ref{sec:voting} share a common part, namely $v_{S_i}$, that represents the evaluation of the contribution of the argument $i$ in forming acceptable extensions. 
For the sake of efficiency, we compute $\pi$ only with respect to those sets $S$ such that either $S$ or $S\cup \{i\}$ is an extension for $\sigma$. In any other case, the value of $v(S \cup \{i\}) - v(S)$ is zero, so we don't need to do the calculation.

We distinguish between two different characteristic functions: $v^{I}_{\sigma,F}(S)$ and $v^{O}_{\sigma,F}(S)$. As stated in Definition~\ref{def:charfun}, the former function takes into account the set of \textit{in} arguments. Given a set of arguments $S$ that does not contain $i$, if $S \cup \{i\}$ is an extension with respect to $\sigma$ and $S$ alone is not, then $i$ brings a positive contribution to the coalition, and its own rank will be higher according to $v^{I}_{\sigma,F}(S)$. On the other hand, the latter function ($v^{O}_{\sigma,F}(S)$) only considers arguments that are labelled \textit{out} by $\sigma$. In detail, $i$ gets a positive value by $\pi$ when $S \cup \{i\}$ is a set of \textit{out} arguments and $S$ alone is not. The set $out(L_\sigma)$ is obtained by computing the sets of arguments that are attacked by the extensions of the semantics $\sigma$.
At this point, each argument of $F$ is associated with the values of the two functions; the resulting structure has the format of an array $[\texttt{arg\_name},\texttt{pi\_in},\texttt{pi\_out}]$, where the three components are: the identifier of the argument, the value of the power index $\pi$ obtained through $v^{I}_{\sigma,F}(S)$, and the value of $\pi$ obtained through $v^{O}_{\sigma,F}(S)$, respectively.

In order to establish the preference relation between two arguments, the PI-based semantics considers the value $\texttt{pi\_in}$ first: the greater the score of an argument with respect to $v^{I}_{\sigma,F}(S)$, the higher its position in the ranking. In case of a tie, i.e., when the value of $\texttt{pi\_in}$ is the same for both the arguments that we want to compare, we perform a further control looking at the value of $v^{O}_{\sigma,F}(S)$. Following the principle that accepted arguments are better than rejected ones, the greater the value of an argument with respect to  $\texttt{pi\_out}$, the lower its position in the ranking.
Consider, for example, two arguments $a$ and $b$, belonging to $F$, with the following evaluations obtained through $\pi$: $[\texttt{a},\texttt{0.2},\texttt{-0.5}]$ and $[\texttt{b},\texttt{0.2},\texttt{-0.4}]$. The value $\texttt{pi\_in}$ is equal for both $a$ and $b$, therefore we proceed to confront the values for $\texttt{pi\_out}$. Since $-0.5 < -0.4$, we have that $a \succ^\pi_F b$. The motivation to this kind of ranking is that, while $a$ and $b$ have the same contribution in forming acceptable extensions, $a$ belongs to fewer sets of \textit{out} arguments, that is, $a$ is defeated less times than $b$. Hence, it is reasonable to prefer $a$ to $b$.

Finally, when all the arguments are sorted according to the semantics and the power index that we have selected, the results of the computation are displayed in the output panel (frame 5 of Figure~\ref{fig:interface}). Along with the overall ranking, we show the \texttt{pi\_in} and \texttt{pi\_out} values of each argument. For providing a visual hint about which arguments are the most preferred and which ones the least, we assign a colour to each node of the AF visualised in the canvas. The assigned colours vary in a greyscale, according to the position of the corresponding argument in the obtained ranking: the lighter the colour, the higher the rank (as depicted in the frame 3 of Figure~\ref{fig:interface}).

%diagramma di flusso per spiegare il processo di ranking

\section{An Example with ConArg}\label{sec:example}
In this section, we provide an example of how the ConArg web interface can be used for dealing with ranking semantics. We show the procedure for obtaining a ranking among the arguments of a given AF through one of the implemented power indexes. We also compare the results for all the different power indexes, highlighting the differences in terms of final ordering of the arguments. For our example, we consider the AF in Figure~\ref{fig:exAF}, that has an initiator (i.e., the argument $a$, which is not attacked by any other argument), a symmetric attack (between $c$ and $d$), a self-attack (in $e$), and a cycle involving $b,d$ and $c$.

\begin{figure}[htb]
	\centering
	\includegraphics[width=4.5cm]{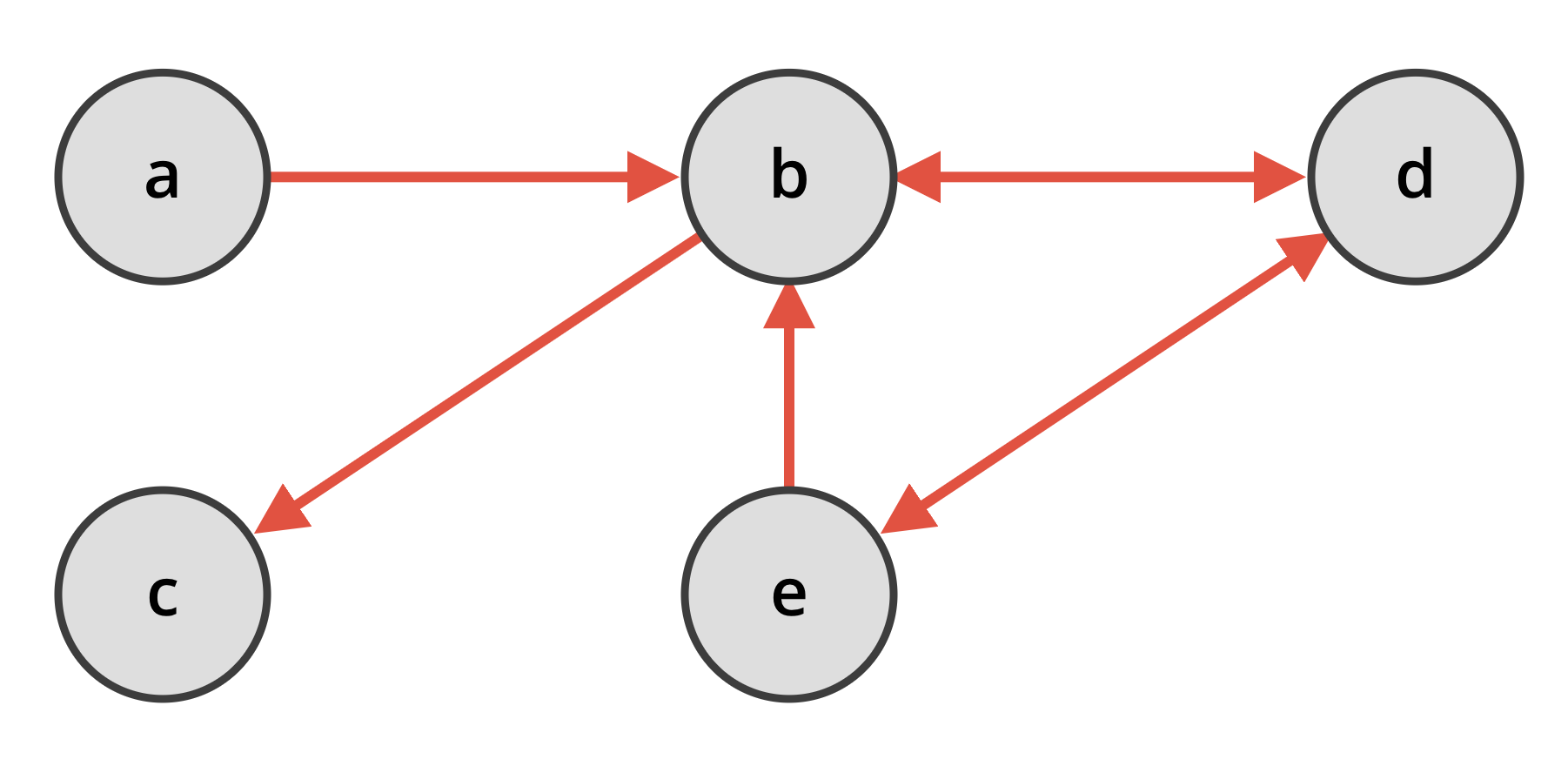}
	\caption{Example of an AF. The sets of extensions for the conflict-free, admissible, complete, preferred and stable semantics are: $\mathit{CF} = \{ \emptyset, \{a\},$ $\{b\},$ $\{c\},$ $\{d\},$ $\{e\},$ $\{a,c\},$ $\{a,d\},$ $\{a,e\},$ $\{c,d\},$ $\{c,e\},$ $\{a,c,d\},$ $\{a,c,e\} \}$, $\mathit{ADM} = \{ \emptyset,$ $\{a\},$ $\{d\},$ $\{e\},$ $\{a,c\},$ $\{a,d\},$ $\{a,e\},$ $\{c,d\},$ $\{c,e\},$ $\{a,c,d\},$ $\{a,c,e\} \}$, $\mathit{COM} = \{ \{a, c\},$ $\{a, c, d\},$ $\{a, c, e\}\}$, and $\mathit{PRE} = \mathit{STB} = \{ \{a, c, d\},$ $\{a, c, e\} \}$, respectively.}
	\label{fig:exAF}
\end{figure}

Given an AF in input, there are two prerequisites for the calculation of the ranking over the arguments. First of all, since the PI-based semantics is parametric to a Dung semantics, this latter must be selected in the semantics panel (frame 2 of Figure~\ref{fig:interface}). Then, we need to choose a power index among the four implemented.
At this point, we can run the computation from the start button.
Below, we report the output provided by ConArg for $\phi$, $\beta$, $\rho$ and $\gamma$ (that correspond to the functions for computing the Shapley Value and the Banzhaf, Deegan-Packel and Johnston Index, respectively), and the semantics conflict-free, admissible, complete and preferred. We omit the stable one since, in this example, it returns the same set of extensions as the preferred.
For each semantics,  the values of the power index obtained with respect to the sets of \textit{in} and \textit{out} arguments are alternated in each row.
Tables~\ref{tab:sv_ex},~\ref{tab:ba_ex},~\ref{tab:dp_ex},~\ref{tab:jo_ex} show the results for the aforementioned indexes.

\begin{table*}[htb]
	\caption{Ranking for the arguments of the AF in Figure~\ref{fig:exAF} obtained through the Shapley Value.}
	\fontsize{5}{9} \selectfont 
	\centering
	\begin{tabular}{c|c|c|c|c|c|c|c}
		\multicolumn{1}{c|}{\textbf{}} & \multicolumn{1}{c|}{\textbf{a}} & \multicolumn{1}{c|}{\textbf{b}} & \multicolumn{1}{c|}{\textbf{c}} & \multicolumn{1}{c|}{\textbf{d}} & \multicolumn{1}{c|}{\textbf{e}} & \multicolumn{1}{c|}{Semantics} & \multicolumn{1}{c}{Ranking} \\ \hline\hline
		$v^{I}_{CF}$ & $-0.05000$ & $-0.46667$ & $-0.05000$ & $-0.21667$ & $-0.21667$ & \multirow{2}{*}{$\phi-CF$} & \multicolumn{1}{c}{\multirow{2}{*}{$a \succ c \succ e \succ d \succ b$}} \\ \cline{1-6}
		$v^{O}_{CF}$ & $-0.35000$ & $0.06667$ & $-0.26667$ & $-0.18333$ & $-0.26667$ &  & \multicolumn{1}{c}{} \\ \hline\hline
		$v^{I}_{ADM}$ & $0.05000$ & $-0.61667$ & $-0.20000$ & $-0.11667$ & $-0.11667$ & \multirow{2}{*}{$\phi-ADM$} & \multicolumn{1}{l}{\multirow{2}{*}{$a \succ d \simeq e \succ c \succ b$}} \\ \cline{1-6}
		$v^{O}_{ADM}$ & $-0.31667$ & $0.10000$ & $-0.31667$ & $-0.23333$ & $-0.23333$ &  & \multicolumn{1}{l}{} \\ \hline\hline
		$v^{I}_{COM}$ & $0.11667$ & $-0.13333$ & $0.11667$ & $-0.05000$ & $-0.05000$ & \multirow{2}{*}{$\phi-COM$} & \multicolumn{1}{l}{\multirow{2}{*}{ $a \simeq c \succ d \simeq e \succ b$}} \\ \cline{1-6}
		$v^{O}_{COM}$ & $-0.11667$ & $0.30000$ & $-0.11667$ & $-0.03333$ & $-0.03333$ &  & \multicolumn{1}{l}{} \\ \hline\hline
		$v^{I}_{PRE}$ & $0.06667$ & $-0.10000$ & $0.06667$ & $-0.01667$ & $-0.01667$ & \multirow{2}{*}{$\phi-PRE$} & \multicolumn{1}{l}{\multirow{2}{*}{$a \simeq c \succ d \simeq e \succ b$}} \\ \cline{1-6}
		$v^{O}_{PRE}$ & $-0.06667$ & $0.10000$ & $-0.06667$ & $0.01667$ & $0.01667$ &  & \multicolumn{1}{l}{} \\ \hline%\hline
		%$v^{I}_{STB}$ & $0.06667$ & $-0.10000$ & $0.06667$ & $-0.01667$ & $-0.01667$ & \multirow{2}{*}{$\phi-STB$} & \multicolumn{1}{l}{\multirow{2}{*}{$a \simeq c \succ d \simeq e \succ b$}} \\ \cline{1-6}
		%$v^{O}_{STB}$ & $-0.06667$ & $0.10000$ & $-0.06667$ & $0.01667$ & $0.01667$ &  & \multicolumn{1}{l}{} \\ \hline
	\end{tabular}\vspace{1em}
	\label{tab:sv_ex}
\end{table*}

\begin{table*}[htb]
	\caption{Ranking for the arguments of the AF in Figure~\ref{fig:exAF} obtained through the Banzhaf Index.}
	\fontsize{5}{9} \selectfont 
	\centering
	\begin{tabular}{c|c|c|c|c|c|c|c}
		\multicolumn{1}{c|}{\textbf{}} & \multicolumn{1}{c|}{\textbf{a}} & \multicolumn{1}{c|}{\textbf{b}} & \multicolumn{1}{c|}{\textbf{c}} & \multicolumn{1}{c|}{\textbf{d}} & \multicolumn{1}{c|}{\textbf{e}} & \multicolumn{1}{c|}{Semantics} & \multicolumn{1}{c}{Ranking} \\ \hline\hline
		$v^{I}_{CF}$ & $-0.06250$ & $-0.68750$ & $-0.06250$ & $-0.31250$ & $-0.31250$ & \multirow{2}{*}{$\beta-CF$} & \multicolumn{1}{c}{\multirow{2}{*}{$a \succ c \simeq e \succ d \succ b$}} \\ \cline{1-6}
		$v^{O}_{CF}$ & $-0.31250$ & $0.06250$ & $-0.18750$ & $-0.06250$ & $-0.18750$ &  & \multicolumn{1}{c}{} \\ \hline\hline
		$v^{I}_{ADM}$ & $0.06250$ & $-0.68750$ & $-0.06250$ & $-0.18750$ & $-0.18750$ & \multirow{2}{*}{$\beta-ADM$} & \multicolumn{1}{l}{\multirow{2}{*}{$a \succ c \succ d \simeq e \succ b$}} \\ \cline{1-6}
		$v^{O}_{ADM}$ & $-0.25000$ & $0.12500$ & $-0.25000$ & $-0.12500$ & $-0.12500$ &  & \multicolumn{1}{l}{} \\ \hline\hline
		$v^{I}_{COM}$ & $0.18750$ & $-0.18750$ & $0.18750$ & $-0.06250$ & $-0.06250$ &  \multirow{2}{*}{$\beta-COM$} & \multicolumn{1}{l}{\multirow{2}{*}{$a \simeq c \succ d \simeq e \succ b$}} \\ \cline{1-6}
		$v^{O}_{COM}$ & $-0.18750$ & $0.18750$ & $-0.18750$ & $-0.06250$ & $-0.06250$ &  & \multicolumn{1}{l}{} \\ \hline\hline
		$v^{I}_{PRE}$ & $0.12500$ & $-0.12500$ & $0.12500$ & $0.00000$ & $0.00000$ & \multirow{2}{*}{$\beta-PRE$} & \multicolumn{1}{l}{\multirow{2}{*}{$a \simeq c \succ d \simeq e \succ b$}} \\ \cline{1-6}
		$v^{O}_{PRE}$ & $-0.12500$ & $0.12500$ & $-0.12500$ & $0.00000$ & $0.00000$ &  & \multicolumn{1}{l}{} \\ \hline%\hline
		%$v^{I}_{STB}$ & $0.12500$ & $-0.12500$ & $0.12500$ & $0.00000$ & $0.00000$ & \multirow{2}{*}{$\beta-STB$} & \multicolumn{1}{l}{\multirow{2}{*}{$a \simeq c \succ d \simeq e \succ b$}} \\ \cline{1-6}
		%$v^{O}_{STB}$ & $-0.12500$ & $0.12500$ & $-0.12500$ & $0.00000$ & $0.00000$ &  & \multicolumn{1}{l}{} \\ \hline
	\end{tabular}\vspace{1em}
	\label{tab:ba_ex}
\end{table*}

\begin{table*}[htb]
	\caption{Ranking for the arguments of the AF in Figure~\ref{fig:exAF} obtained through the Deegan-Packel Index.}
	\fontsize{5}{9} \selectfont 
	\centering
	\begin{tabular}{c|c|c|c|c|c|c|c}
		\multicolumn{1}{c|}{\textbf{}} & \multicolumn{1}{c|}{\textbf{a}} & \multicolumn{1}{c|}{\textbf{b}} & \multicolumn{1}{c|}{\textbf{c}} & \multicolumn{1}{c|}{\textbf{d}} & \multicolumn{1}{c|}{\textbf{e}} & \multicolumn{1}{c|}{Semantics} & \multicolumn{1}{c}{Ranking} \\ \hline\hline
		$v^{I}_{COM}$ & $0.50000$ & $0.00000$ & $0.50000$ & $0.00000$ & $0.00000$ &  \multirow{2}{*}{$\rho-COM$} & \multicolumn{1}{l}{\multirow{2}{*}{$a \simeq c \succ b \simeq d \simeq e$}} \\ \cline{1-6}
		$v^{O}_{COM}$ & $0.00000$ & $0.00000$ & $0.00000$ & $0.00000$ & $0.00000$ &  & \multicolumn{1}{l}{} \\ \hline\hline
		$v^{I}_{PRE}$ & $0.33333$ & $0.00000$ & $0.33333$ & $0,16667$ & $0,16667$ & \multirow{2}{*}{$\rho-PRE$} & \multicolumn{1}{l}{\multirow{2}{*}{$a \simeq c \succ d \simeq e \succ b$}} \\ \cline{1-6}
		$v^{O}_{PRE}$ & $0.00000$ & $0.66667$ & $0.00000$ & $0.33333$ & $0.33333$ &  & \multicolumn{1}{l}{} \\ \hline%\hline
		%$v^{I}_{STB}$ & $0.33333$ & $0.00000$ & $0.33333$ & $0,16667$ & $0,16667$ & \multirow{2}{*}{$\rho-STB$} & \multicolumn{1}{l}{\multirow{2}{*}{$a \simeq c \succ d \simeq e \succ b$}} \\ \cline{1-6}
		%$v^{O}_{STB}$ & $0.00000$ & $0.66667$ & $0.00000$ & $0.33333$ & $0.33333$ &  & \multicolumn{1}{l}{} \\ \hline\hline
	\end{tabular}\vspace{1em}
	\label{tab:dp_ex}
\end{table*}

\begin{table*}[htb]
	\caption{Ranking for the arguments of the AF in Figure~\ref{fig:exAF} obtained through the Johnston Index.}
	\fontsize{5}{9} \selectfont 
	\centering
	\begin{tabular}{c|c|c|c|c|c|c|c}
		\multicolumn{1}{c|}{\textbf{}} & \multicolumn{1}{c|}{\textbf{a}} & \multicolumn{1}{c|}{\textbf{b}} & \multicolumn{1}{c|}{\textbf{c}} & \multicolumn{1}{c|}{\textbf{d}} & \multicolumn{1}{c|}{\textbf{e}} & \multicolumn{1}{c|}{Semantics} & \multicolumn{1}{c}{Ranking} \\ \hline\hline
		$v^{I}_{CF}$ & $0.00000$ & $-3.16667$ & $0.00000$ & $-2.50000$ & $-2.50000$ & \multirow{2}{*}{$\gamma-CF$} & \multicolumn{1}{c}{\multirow{2}{*}{$a \succ c \succ e \succ d \succ b$}} \\ \cline{1-6}
		$v^{O}_{CF}$ & $-2.50000$ & $1.00000$ & $-2.00000$ & $-0.50000$ & $-1.50000$ &  & \multicolumn{1}{c}{} \\ \hline\hline
		$v^{I}_{ADM}$ & $1.00000$ & $-6.16667$ & $0.00000$ & $-1.50000$ & $-1.50000$ & \multirow{2}{*}{$\gamma-ADM$} & \multicolumn{1}{l}{\multirow{2}{*}{$a \succ c \succ d \simeq e \succ b$}} \\ \cline{1-6}
		$v^{O}_{ADM}$ & $-2.00000$ & $2.00000$ & $-2.00000$ & $-1.00000$ & $-1.00000$ &  & \multicolumn{1}{l}{} \\ \hline\hline
		$v^{I}_{COM}$ & $1.50000$ & $-1.16667$ & $1.50000$ & $-0.50000$ & $-0.50000$ &  \multirow{2}{*}{$\gamma-COM$} & \multicolumn{1}{l}{\multirow{2}{*}{$a \simeq c \succ d \simeq e \succ b$}} \\ \cline{1-6}
		$v^{O}_{COM}$ & $-2.00000$ & $3.00000$ & $-2.00000$ & $-1.00000$ & $-1.00000$ &  & \multicolumn{1}{l}{} \\ \hline\hline
		$v^{I}_{PRE}$ & $0.66667$ & $-0.66667$ & $0.66667$ & $-0.16667$ & $-0.16667$ & \multirow{2}{*}{$\gamma-PRE$} & \multicolumn{1}{l}{\multirow{2}{*}{$a \simeq c \succ d \simeq e \succ b$}} \\ \cline{1-6}
		$v^{O}_{PRE}$ & $-1.00000$ & $1.00000$ & $-1.00000$ & $-0.50000$ & $-0.50000$ &  & \multicolumn{1}{l}{} \\ \hline%\hline
		%$v^{I}_{STB}$ & $0.66667$ & $-0.66667$ & $0.66667$ & $-0.16667$ & $-0.16667$ & \multirow{2}{*}{$\gamma-STB$} & \multicolumn{1}{l}{\multirow{2}{*}{$a \simeq c \succ d \simeq e \succ b$}} \\ \cline{1-6}
		%$v^{O}_{STB}$ & $-1.00000$ & $1.00000$ & $-1.00000$ & $-0.50000$ & $-0.50000$ &  & \multicolumn{1}{l}{} \\ \hline\hline
	\end{tabular}\vspace{1em}
	\label{tab:jo_ex}
\end{table*}

We now analyse the differences between the obtained rankings, following two levels of detail: we first compare, for each power index, the ranking obtained for all the Dung semantics. Then, for each Dung semantics, we consider the ranking obtained with respect to the different power indexes.

In this example, the Shapley Value (Table~\ref{tab:sv_ex}), provides a rankling without indifferences when the conflict-free semantics is considered. While $\phi-com$ , $\phi-pre$ and $\phi-stb$ return the same output, where in particular $c \succ d$ and $c \succ e$, the ranking for the admissible semantics gives an opposite interpretation, that is $d \succ c$ and $e \succ c$. This happens because both $\{d\}$ and $\{e\}$ are admissible extensions, while $\{c\}$ is not.
Hence, when the admissible semantics is taken into account, $d$ and $e$ are better arguments than $c$.

When Banzhaf Index is used (Table~\ref{tab:ba_ex}), such an inversion of preferences never occurs: there is no semantics for which $d \succ c$ or $e \succ c$. Looking at the formulas of the Shapley Value $\phi$ (Equation~\ref{eq:sv}) and the Banzhaf Index $\beta$ (Equation~\ref{eq:ban}), we can see that the only difference is the factor by which the gain $v(S \cup \{i\}) - v(S)$ is multiplied. Contrary to Shapley, Banzhaf does not consider the order in which the coalitions form; since the acceptability of the arguments does not depend on how the extensions are formed, $\beta$ produces more consistent results and, therefore, is a more appropriate index to be used for building a ranking-based semantics.

Using the Deegan-Packel Index for computing the ranking with respect to the conflict-free and the admissible semantics is not meaningful. Indeed, for such semantics, the empty set $\emptyset$ is always an extension, and it also represents the only minimal winning coalitions. Since $\rho$ relies on the set of minimal winning coalitions $M(v)$, when $\emptyset$ is the only element of $M(v)$, all the arguments receive a value of $0$, according to Equation~\ref{eq:dp}. For this reason, we omit to include $\rho - CF$ and $\rho - ADM$ in Table~\ref{tab:dp_ex}.

%discussione DP e JI

The ranking obtained through all power indexes share some common features, that we discuss below.
The argument $a$, that is not attacked by any other, is always in the first position of the rank, for every power index.
Consequently, the argument $b$, that is attacked by $a$, always results to be the worst argument in the AF, excepted for indifferences.
For the complete, preferred and stable semantics, the ranking does not distinguish between $a$ and $c$, and between $d$ and $e$. Indeed, the set ${a,c}$ corresponds to the grounded semantics, that is $a$ and $c$ are equally ``important'' and should be evaluated the same. Similarly, $e$ and $d$, that only appear in two distinct maximal admissible sets, receive the same value from all the power indexes.
Finally, since the extensions of the preferred and the stable coincide, these two semantics always provide the same final ranking.

\section{Conclusion and Future Work}\label{sec:conclusion}
We have presented an online tool capable of dealing with ranking-based semantics. The tool implements the definition of the PI-based semantics~\cite{bistarelli_cooperative-game_2018} with respect to four power indexes, namely the Shapley Value and the Banzhaf, Deegan-Packel and Johnston Index. Differently from other ranking-based semantics defined in the literature, our approach allows for distributing preferences among arguments taking into account classical Dung/Caminada semantics. In this way, we obtain a more accurate ranking with respect to the desired acceptability criterion. We have also provided an applicative example in which an AF is studied from the point of view of the different power indexes that can be used for extracting the ranking over the arguments; the tool can be used for studying properties of the ranking-based semantics and in particular the PI-based ones, where also notions from cooperative games converge.

In the future, we plan to implement other indexes in the tool, or combinations of them. As a starting point, we could use the work in~\cite{DBLP:journals/corr/abs-1812-05808}, where various power indexes are grouped according to some criteria that qualify them for certain applications. In particular, we may consider the Public Good index, that is said to detect ``special games''.
We aim to understand which ranking properties (or families of them, i.e., local or global) listed in \cite{bonzon_comparative_2016} such indexes can successfully capture. With the comparison of different indexes, we aim to determine if there is a link between ties on rankings and the possible resolution of ambiguities.
So far, we have only captured properties that are local to an argument, i.e., they can be checked by inspecting the immediate neighbourhood of an argument. Global properties derive, instead, from the whole framework structure (e.g., full attacking or defending paths), and could be exploited for further refining the ranking returned by our semantics. We are also interested in extending our work on weighted AFs, where a different notion of defence is used.

%complexity
As a further step, we want to conduct a thorough study on the computational complexity of our semantics. We deal with two kind of problems: the computation of sets of extension, and of the power indexes.
For what concerns argumentation, the only problems for which efficient algorithms exist are, according to~\cite{DBLP:books/sp/09/DunneW09}: i) deciding the credulous/sceptical acceptability of an argument with respect to the grounded semantics, and ii) verifying if an extension is conflict-free, admissible, or stable. All the other problems are intractable.
On the other hand, as shown in~\cite{Matsui00asurvey}, computing the power indexes $\phi_i(v)$, $\beta_i(v)$ and $\rho_i(v)$ is $NP$-hard.
Due to the specific setting in which we use the power indexes, we think that the used formulas can be simplified, so to make the computation faster.

Another direction we plan to investigate concerns the ranking function we use for evaluating the arguments. Similarly to what is done in~\cite{derks_shapley_1993} for studying coalitions with particular properties, we want to restrict the set of possible extensions by considering only the subsets of $S$ that are in a given semantics, that is to exclude arguments that are not even credulously accepted. For instance, we could devise a PI-based semantics where the arguments are evaluated with respect to the stable semantics and the only coalitions $S$ to be taken into account are the admissible ones.
Lastly, we would like to extend the ConArg library written in C with the ranking functions for all the indexes we proposed, so to obtain a more efficient implementation for the PI-based semantics.

\bibliographystyle{plain}
\bibliography{ICTAI2019TR}

\end{document}